\documentclass[12pt,letterpaper]{article}
\usepackage[utf8]{inputenc}
\usepackage{amsmath}
\usepackage{mathptmx, courier,textcomp}
\usepackage{amsfonts}
\usepackage{amssymb}
\usepackage{graphicx}
\usepackage{subcaption}
\usepackage{cite}
\usepackage{booktabs}
\usepackage{tablefootnote}
\usepackage{boldline}
\usepackage{caption}
\usepackage{algorithm}
\usepackage{algorithmic}
\usepackage{slashbox}
\usepackage{epstopdf}
\usepackage[top=1in,bottom=1in,left=1in,right=1in]{geometry} 
\usepackage{float}
\usepackage{hyperref}
\usepackage{booktabs}
\usepackage[export]{adjustbox}
\usepackage{float}
\usepackage{multirow}
\usepackage{multicol}
\usepackage{pifont}
\usepackage[ ]{authblk}
\usepackage{chemarrow}
\usepackage{graphicx}
\usepackage{booktabs}
\usepackage{lmodern}
\usepackage{svg}
\usepackage{color,soul}
\usepackage{amsmath}
\usepackage{slashbox}
\usepackage{multirow} 
\usepackage{amsmath}
\usepackage{booktabs} 
\usepackage{float}
\usepackage{orcidlink}
\usepackage{algorithm}
\usepackage{algorithmic}
\usepackage[english]{babel}
\usepackage{amsthm}
\usepackage{threeparttable}
\usepackage{comment}
\usepackage{makecell}
\usepackage{orcidlink}
\usepackage[normalem]{ulem}
\newtheorem{theorem}{Theorem}
\newtheorem{definition}{Definition}
\newtheorem{prop}{Proposition}

\newtheorem{Rem}{Remark}
\UseRawInputEncoding
\hypersetup{
    colorlinks=true,
    linkcolor=blue,
    filecolor=magenta,      
    urlcolor=blue,
    pdftitle={Overleaf Example},
    pdfpagemode=FullScreen,
}
\usepackage{multicol}

\begin{document}
%
% --------------------------------------------- Start Title
%
\title{DP-CDA: An Algorithm for Enhanced Privacy Preservation in Dataset Synthesis Through Randomized Mixing}
%
% --------------------------------------------- End Title
% 

%
% ---------------------------------------------  Start authors
%
\author[1,$\ddagger$]{Utsab Saha\orcidlink{0000-0003-2106-8648}}
\author[1,$\ddagger$]{Tanvir Muntakim Tonoy\orcidlink{0009-0005-9915-5014}}
\author[1,*]{Hafiz Imtiaz\orcidlink{0000-0002-2042-5941}}

\affil[1]{\small{Department of Electrical and Electronic Engineering\\
Bangladesh University of Engineering Technology, 
Dhaka 1205, Bangladesh}}
% \affil[2]{\small{Department of Computer Science and Engineering\\
% BRAC University\\
% Dhaka 1212, Bangladesh}}
\affil[$\ddagger$]{\small{\it{These authors contributed equally to this work}}}
\affil[*]{\small{\it{hafizimtiaz@eee.buet.ac.bd}}}

%
% ---------------------------------------------  end authors 
%

% ---------------------------------------------
\date{}
\maketitle
\begingroup
\renewcommand\thefootnote{}\footnote{This manuscript has been published in the \textit{SECURITY AND PRIVACY} by \textbf{Wiley}.}\addtocounter{footnote}{-1}
\endgroup

\vspace{-45 pt}
% -----------------------------------------
% --------------------------------------------
\sloppy
%
% --------------------------------------------- Start abstract
%
\begin{abstract}
In recent years, the growth of data across various sectors, including healthcare, security, finance, and education, has created significant opportunities for analysis and informed decision-making. However, these datasets often contain sensitive and personal information, which raises serious privacy concerns. It has been shown in multiple works that a person's identity is intertwined with their data, even if the data is anonymized. Due to this lack of separation between a person's identity and their information, the patterns associated with an individual's information can uniquely identify them. Protecting individual privacy is crucial, yet many existing machine learning and data publishing algorithms struggle with high-dimensional data, facing challenges related to the trade-off between computational efficiency and privacy. To address these challenges, we introduce an effective data publishing algorithm \emph{DP-CDA}. Our proposed algorithm generates synthetic data by randomly mixing the privacy-sensitive data in a class-specific manner and inducing carefully tuned randomness to ensure formal privacy guarantees. Our comprehensive privacy accounting shows that the proposed DP-CDA provides a stronger privacy guarantee compared to existing methods, allowing for better utility while maintaining a stricter level of privacy. To evaluate the effectiveness of DP-CDA, we examine the accuracy of predictive models trained on the synthetic data, which serves as a measure of dataset utility. Importantly, we identify an optimal order of mixing that balances privacy-utility trade-off. Our results indicate that synthetic datasets produced using the DP-CDA can achieve superior utility compared to those generated by conventional data publishing algorithms, even when subject to the same privacy requirements.

\end{abstract}


%
% --------------------------------------------- End abstract
%
% \pagebreak
%
%%%%%%%%%%%%%%%%%%%%%%%%%%%%% body  %%%%%%%%%%%%%%%%%%%%%%%%%%%%%%%%
%
\section{Introduction}
\label{sec:intro}
In the realm of machine learning (ML) or deep learning (DL) applications, it is widely acknowledged that the dataset is of utmost importance, since it dictates the performance of the model. Oftentimes, the dataset contains sensitive information, such as medical records, personal photos/information, or proprietary data. Organizations, including private companies, government agencies, and hospitals, routinely handle substantial volumes of sensitive personal information about their clients, customers, or patients. This underscores the importance of implementing systems and methodologies that can effectively analyze the data while preserving privacy and ensuring the confidentiality of individuals. Using standard ML techniques to train a model can potentially lead to data breaches~\cite{shokri2017membership}. It should be noted here that only anonymization can not provide any privacy in the presence of auxiliary information~\cite{imtiaz2021cape}. More specifically, it is important to be aware of ``composition attacks'', where malicious actors can combine their algorithm outputs with other information to identify individuals in the dataset~\cite{ganta2008composition}. For instance, an adversary may have the opportunity to access publicly available records, including voting polls~\cite{sweeney2002k}. This is due to the lack of separation between an individual's identity and their data, as the patterns associated with a person's information can be uniquely identifying~\cite{sarwate2013signal}. It highlights the importance of considering data privacy and security measures to protect the identities of the individuals. A prudent approach involves considering the creation of a ``distilled version'' of the original sensitive dataset and subsequently, training the ML model exclusively using that version to protect the sensitive information~\cite{fung2010privacy, zhu2017differentially, fukuchi2017differentially}.

A mathematically rigorous and cryptographically motivated approach for preserving privacy, differential privacy, has gained significant attention among researchers, corporations, and users alike~\cite{dwork2006calibrating}. It quantifies the risk of privacy breaches using a parameter $\epsilon$, which governs the extent to which the outcome of a privacy-preserving algorithm can vary when an individual's data is included or excluded from the dataset. A smaller $\epsilon$ reduces the ability of adversaries to infer information about individuals in the dataset, enhancing overall privacy protection. The concept of differential privacy (DP), introduced in~\cite{dwork2008differential}, has gained widespread acceptance in various fields, including healthcare~\cite{liu2024survey, fuladi2025reliable}, ML, or DL applications~\cite{pan2024differential, demelius2025recent}, recommender systems~\cite{fang2022differentially, deng2025differentially}, etc. To ensure DP for an ML model training, some form of randomness is introduced into the training algorithm pipeline. This evidently affects the model's utility, giving raise to a \emph{privacy-utility trade-off}. 

In addition to differentially-private model training, synthesizing data that preserves privacy and the ``structures or geometry'' of the original sensitive data has been investigated. Such synthetic data is used for ML model training, and the resulting model is provably differentially private~\cite{dwork2008differential}. While several methods for differentially private data publishing have been proposed for ML and DL applications, many are not suitable for high-dimensional data due to significant computational demands and utility penalties~\cite{lee2019synthesizing}. These algorithms can be categorized as: (i) Local perturbation algorithms, (ii) Differentially private machine learning algorithms, (iii) Mixing algorithms, and (iv) Generative Adversarial Network (GAN)-based algorithms.

Local perturbation is a straightforward approach that adds noise to the data. However, to achieve a meaningful level of privacy, a large amount of additive noise is required, making it challenging to achieve a good privacy-utility trade-off.
On the other hand, mixing algorithms avoid many of the issues. However, a common problem with previous approaches is the implementation of noiseless mixing, which tends to result in weaker privacy guarantees. In this context, the work of K. Lee et al.~\cite{lee2019synthesizing} is particularly relevant as it was one of the first to demonstrate the effectiveness of high-order mixing with additive noise for privacy-preserving synthetic data generation. While their approach established a solid foundation, its privacy bound remains dependent on the data dimension. This emphasizes the importance of developing more efficient algorithms capable of handling large, complex datasets while still ensuring strict privacy protection.

\begin{figure}[t]
    \centering
    \includegraphics[width=0.8\textwidth]{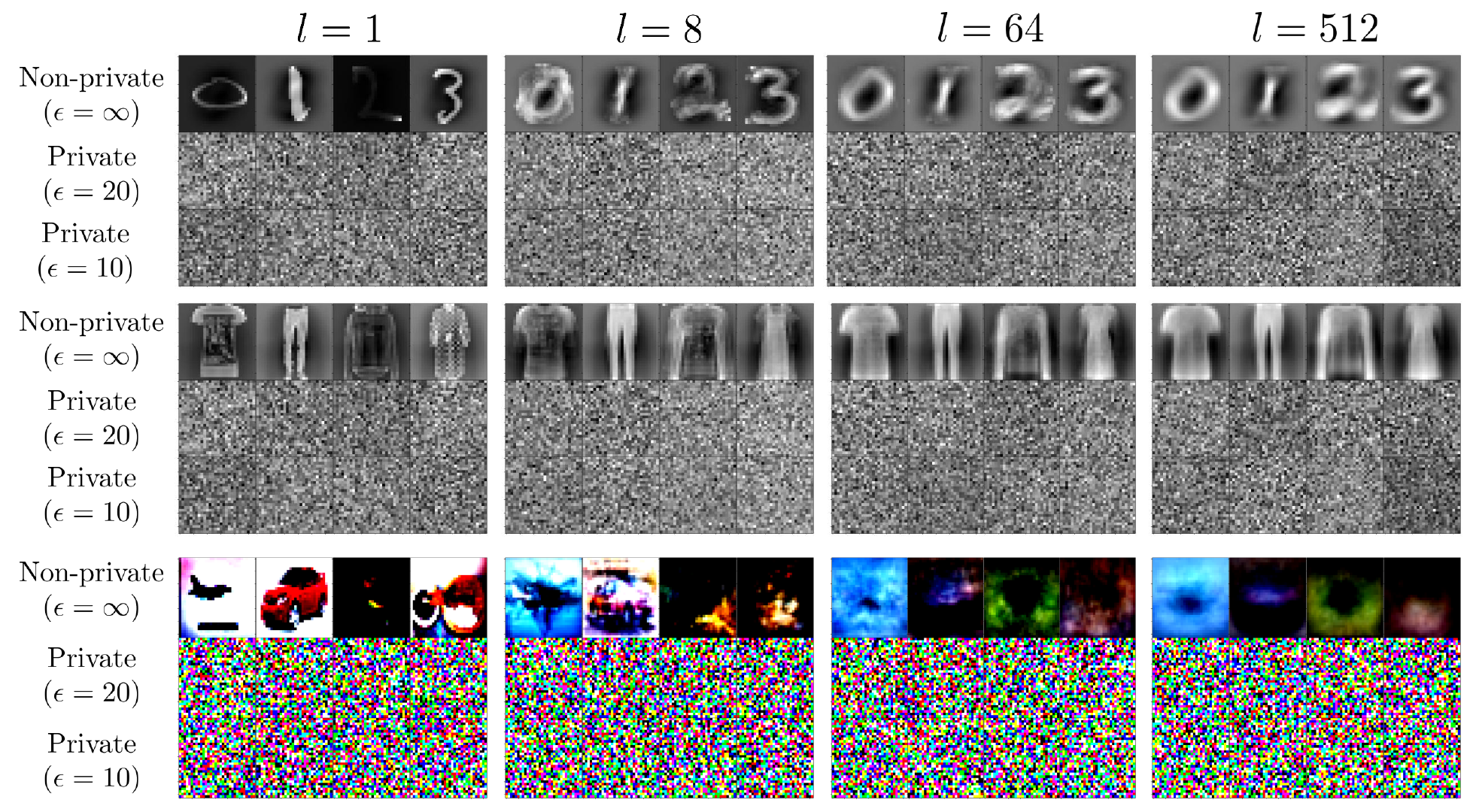}
    \caption{Differentially private datasets generated from the MNIST, FashionMNIST, and CIFAR10 datasets, each incorporating varying order of mixture ($l$). The parameter $\epsilon$ refers to the DP level; a lower value of $\epsilon$ signifies a greater degree of privacy protection.}
    \label{fig:fig1}
\end{figure}

Motivated by this, we present a synthetic privacy-preserving data publishing algorithm -- \textit{Differentially Private Class-Centric Data Aggregation (DP-CDA)}. More specifically, we (uniformly) randomly select $l$ data samples from a specific class for mixing and then induce Gaussian noise (parameterized by variance terms $\sigma_x$, $\sigma_y$) to generate each synthetic data sample. This distilled dataset is subsequently utilized in the conventional training process of ML and DL frameworks. A few samples of the synthetic data generated by the proposed algorithm for different privacy levels are shown in Figure~\ref{fig:fig1}, whereas detailed results are presented in Section~\ref{sec:res}. We summarize our main contributions below:
\begin{figure}[H]
    \centering
 % Encloses the image in a box
        \includegraphics[width =0.5\linewidth]{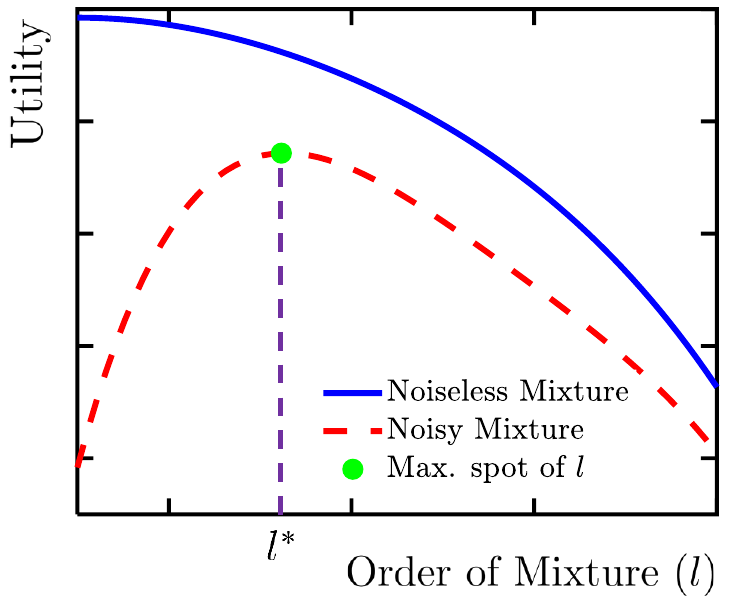}
    \caption{Utility as a function of the order of mixture $l$.}
    \label{fig:fig2}
\end{figure}
\begin{itemize}
    \item We analyze the synthetic data generation algorithm to achieve a tighter overall privacy guarantee compared to existing works.
    \item We investigate the effect of order of mixture \( l \) on the performance of the ML model trained on the generated synthetic data. More specifically, we investigated the performance of our proposed approach on four real datasets (among which three are image datasets and one is tabular dataset). We empirically demonstrate that for a given dataset and privacy level, there exists an \emph{optimal} $l^*$ for which the model performance peaks, as illustrated in the toy plot in Figure~\ref{fig:fig2}.
    \item Through our tighter theoretical analysis and experimental validation, we demonstrate that the synthetic data generation algorithm can achieve a given level of privacy while ensuring better utility compared to existing approaches.
\end{itemize}
The remainder of this paper is organized as follows. In Section~\ref{sec:rel_work}, we review relevant literature and highlight key developments in the field. In Section~\ref{sec:Definitions}, we outline the necessary background and formally define the problem setting. We describe the proposed algorithm in detail in Section~\ref{sec:proposed-method}. In Section~\ref{sec:res}, we present the experimental results to demonstrate the effectiveness of the proposed approach. In Section~\ref{sec:limit}, we discuss the limitations of our method and outline potential directions for future work. Finally, in Section~\ref{sec:conc}, we conclude the paper with a summary of our main contributions.

\section{Related Works}
\label{sec:rel_work}
Over the years, for publishing differentially private data, numerous algorithms have been developed~\cite{xiao2010differential, zhang2017privbayes}. For a comprehensive analysis of these algorithms, we refer the reader to the works of Zhu et al.~\cite{zhu2017differentially}. Nevertheless, these algorithms face challenges in scaling with larger data sets -- the computational complexity increases significantly with high-dimensional data, often exhibiting exponential growth~\cite{lee2019synthesizing}. As a result, many of these algorithms can become impractical for deep learning applications, particularly those that involve high-dimensional datasets. Local Perturbation is a widely adopted algorithm for data publishing within the framework of DP~\cite{agrawal2000privacy, agrawal2005privacy, agrawal2005framework, mishra2006privacy}. This technique involves the addition of noise from a specific distribution to data points, allowing the generated data to remain consistent with the original data domain. This compatibility makes it suitable for use with existing learning algorithms. However, to ensure meaningful privacy, noise with high variance must be introduced, which adversely affects model performance. In contrast, random projection reduces high-dimensional data to a lower-dimensional subspace prior to inducing noise~\cite{kenthapadi2012privacy, xu2017dppro}. This method results in a new data domain, which can limit the applicability of models that are specific to a given domain.  

There is a considerable amount of work on the application of differentially private algorithms to various machine learning problems, including regression~\cite{zhang2012functional, chaudhuri2008privacy, zheng2017collect}, online learning~\cite{agarwal2017price}, graphical models~\cite{bernstein2017differentially}, empirical risk minimization~\cite{chaudhuri2011differentially, song2013stochastic, bassily2014private, chaudhuri2011differentially, imtiaz2021cape, naima2023fmdp}, and deep learning~\cite{abadi2017protection, abadi2016deep, papernot2016semi}. The objectives of these algorithms differ slightly from those of the data publishing algorithms mentioned earlier. While differentially private data publishing algorithms focus on designing synthetic datasets at the input end of the pipeline, differentially private algorithms aim to create learning algorithms that yield differentially private models at the output end. Any model trained on a dataset that is differentially private will also maintain DP~\cite{dwork2008differential}; the goals of these learning algorithms represent a more relaxed version of the data publishing algorithms. Unfortunately, these differentially private algorithms do not always guarantee data privacy in adverse scenarios. If an attacker gains access to the input side of the pipeline, the entire algorithm may be compromised, thus undermining the security of user data.

In the past few years, several methods have been proposed for differentially private data publishing algorithms using mixtures. Karakus et al. illustrated that it is possible to effectively train a straightforward linear model by utilizing mixtures that are free from noise~\cite{karakus2017straggler}. The exploration of learning nonlinear models with mixtures has also been addressed in previous studies~\cite{tokozume2017learning, tokozume2018between, zhang2018mixup, inoue2018data, lee2018sgd}. For instance, Tokozune et al. demonstrated that sound recognition models can be trained using mixtures of audio signals~\cite{tokozume2017learning}. Similarly, other studies have indicated that it is possible to train an image classification model using image mixtures~\cite{tokozume2018between, zhang2018mixup, inoue2018data}. However, these studies primarily focused on noiseless mixtures involving only two or three data points, whereas Lee et al. demonstrated how to train deep neural networks with a high order of mixture, incorporating added noise to enhance privacy~\cite{lee2019synthesizing}. 

Recently, Zhang et al. introduced a differentially-private data synthesis algorithm with correlation based on latent factor models, which proves to be efficient for mixed-type data~\cite{zhang2024differentially}. Additionally, data publishing algorithms have been automated by incorporating a deep learning framework~\cite{kumar2024differential}. More specifically, the integration of GAN for producing differentially private datasets has gained much popularity~\cite{cao2021don, xu2019ganobfuscator}. Although GAN-generated synthetic images offer DP guarantees and maintain the utility of the original private images, they often appear visually similar to those private images. This similarity poses a challenge in fulfilling a key privacy requirement -- visual dissimilarity. Furthermore, training a GAN can be difficult due to vanishing gradients, mode collapse, training instability, and convergence failures. 

\begin{table}[t]
\centering
\caption{Comparison of existing differentially private data publishing methods}
\label{tab:dp_comparison}
\resizebox{1.0\textwidth}{!}{
\begin{tabular}{p{3.2cm} p{2.3cm} p{3cm} p{2.1cm} p{2cm} p{1.7cm} p{2.2cm}}
\hline
\hline
\textbf{Method} & \textbf{Approach Type} & \textbf{Noise Addition Stage} & \textbf{Utility on High-Dim Data} & \textbf{Visual Dissimilarity} & \textbf{Scalability} & \textbf{Decentralized Support} \\
\hline
Local Perturbation~\cite{agrawal2000privacy} & Input Perturbation & Directly on original data & Low & Preserved & High & Limited \\

Random Projection~\cite{xu2017dppro} & Input Perturbation & After dimensionality reduction & Low & Low & Moderate & Limited \\

DP-SGD~\cite{abadi2016deep} & Output Perturbation & On gradient updates & High & Not applicable & High & Yes \\

GAN-based~\cite{cao2021don, xu2019ganobfuscator} & Generative Model & Internal in GAN training & High & Low & Moderate & Possible \\

DP-Mix~\cite{lee2019synthesizing} & Data Mixing & On mixed samples & Moderate to High & Moderate & High & Limited \\

DP-CDA (Proposed) & Data Mixing & On mixed samples & High & High & High & Yes \\
\hline
\end{tabular}
}
\end{table}

It is evident from existing works that mixing algorithms for differentially private data publishing are particularly effective in balancing computational efficiency and privacy preservation. Motivated by this, we focus on enhancing existing mixing-based approaches to address their limitations. More specifically, we consider the problem of releasing synthetic data for training an image classification model while ensuring $(\epsilon, \delta)$-DP guarantees and providing utility close to that of conventional DP training. We note that the proposed synthetic data-generating mechanism is computationally simple and can be adapted for high-dimensional and decentralized data. Furthermore, we demonstrate that our method necessitates considerably less noise addition compared to existing mixing-based data publishing techniques. As a result, our proposed approach effectively preserves the utility of the released data while ensuring robust privacy protection. 
To contextualize our work among various methods for publishing private data, we provide a structured comparison in Table~\ref{tab:dp_comparison}. Here, we compare the proposed DP-CDA algorithm to existing representative approaches -- including input/output perturbation, random projection, GAN-based synthesis, and earlier mixing-based techniques, on key factors such as approach type, noise injection process, utility support for high-dimensional data, visual dissimilarity, scalability, and decentralized support. While previous approaches frequently struggled to being computationally light-weight, provide strong privacy guarantees with high utility, scalability, and support operation in decentralized settings, the proposed DP-CDA achieves an excellent balance across all of these dimensions. This clearly demonstrates practical suitability for scalable and privacy-preserving synthetic data generation.

\section{Background and Problem Formulation}\label{sec:Definitions}
\textbf{Notation. }We use lower-case bold-faced letters (e.g., $\mathbf{x}$) for vectors, upper-case bold-faced letters (e.g., $\mathbf{X}$) for matrices, and unbolded letters (e.g., $N$ or $n$) for scalars. The set of $n$ data points $(\mathbf{x}_i, y_i)$ is represented as $\mathbb{D} = \{(\mathbf{x}_1, y_1), (\mathbf{x}_2, y_2), \ldots, (\mathbf{x}_n, y_n)\} \triangleq \{ x_i, y_i \}_{i=1}^n$. We use $\| \cdot \|_2$ and $\| \cdot \|_F$ to denote the $\ell_2$-norm of a vector and the Frobenius norm of a matrix, respectively.\\

We reviewed some definitions, theorems, and propositions, which are necessary for our problem formulation, according to~\cite{imtiaz2021cape, swapnil2023, naima2023fmdp} in \ref{sec:appen}.\\

\noindent\textbf{Problem Formulation. }As mentioned before, numerous modern ML algorithms are trained on privacy-sensitive data, and the model parameters can cause significant privacy leakage. In this work, we focus on generating synthetic data that i) satisfies formal and strict DP guarantees; and ii) can be utilized to train ML models. More specifically, given a dataset $\mathbb{D} = \{ (\mathbf{x}_i, \mathbf{y}_i) \}_{i=1}^n$, where each $\mathbf{x}_i \in \mathbb{R}^{d_x}$ is a feature vector and $\mathbf{y}_i \in \mathbb{R}^{d_y}$ is the corresponding label, the goal is to construct a new dataset $\mathbb{D}'$ that ensures privacy, while maintaining high utility for downstream ML tasks. 
% A randomized mechanism $\mathcal{M}$ is said to provide $\epsilon$-differential privacy if for any two \emph{neighboring} datasets $\mathbb{D}$ and $\mathbb{D}'$ (i.e., differing by at most one record), and for all possible outputs $\mathbb{S}$ the following is satisfied:
% $$\Pr[\mathcal{M}(\mathbb{D}) \in \mathbb{S}] \leq \exp(\epsilon) \Pr[\mathcal{M}(\mathbb{D}') \in \mathbb{S}].$$
% Here, $\epsilon > 0$ is a privacy budget that quantifies the level of privacy protection, with smaller values indicating stronger privacy guarantees. 
The new dataset $\mathbb{D}'$ should enable effective training of a machine-learning model $f(\mathbf{w})$, yielding predictive performance comparable to that achieved with the original dataset $\mathbb{D}$. This requirement can be formulated as a utility constraint: 
$$\text{Utility}(f, \mathbb{D}') \geq \theta \text{Utility}(f, \mathbb{D}),$$ 
where, $\text{Utility}(f, \mathbb{D}')$ is a metric (e.g., accuracy) that reflects the performance of model $f(\mathbf{w})$ when trained on $\mathbb{D}'$, and $\theta \in (0,1]$ is a predefined utility threshold indicating acceptable performance relative to training on $\mathbb{D}$. As mentioned before, our objective of this work is to design a transformation mechanism $\mathcal{M}$ that maximizes privacy protection, while simultaneously meeting the utility constraint for effective model performance. This can be expressed as:
\begin{align*}
    &\text{minimize} \quad \epsilon \\ 
    &\text{subject to} \quad \text{Utility}(f, \mathbb{D}') \geq \theta \text{Utility}(f, \mathbb{D}).
\end{align*}

\section{Proposed Privacy-Preserving Synthetic Data Generation Algorithm (DP-CDA)}\label{sec:proposed-method}
In the following, we describe our proposed DP-CDA -- a privacy-preserving synthetic data generation algorithm, in detail with a similar setup as~\cite{lee2019synthesizing}. Let us assume we have a dataset $\mathbb{D} = \{ (\mathbf{x}_i, y_i) \}_{i=1}^N$. We can form the \emph{design matrix} $\hat{\mathbf{X}} \in \mathbb{R}^{N \times d_x}$ and $\mathbf{y} \in \mathbb{R}^{N}$, where the $i$-th sample $\mathbf{x}_i \in \mathbb{R}^{d_x}$ is the $i$-th row in $\hat{\mathbf{X}}$, and $y_i \in \{1, 2, 3, ..., K\}$ is the corresponding class label. We intend to synthesize $T$ new data points by mixing $l$ data points from the same class. We first perform a feature-wise $z$-score normalization of the data points $\mathbf{x}_i$ as
\[
x_{ij} = \frac{x_{ij} - \mu_j}{\sigma_j} \quad \forall j \in \{1, 2, \ldots, d_x\},
\]
where $x_{ij}$ is the $j$-th entry of the $d_x$-dimensional feature vector $\mathbf{x}_i$, $\mu_j = \frac{1}{N} \sum_{i=1}^{N} x_{ij}$, and $\sigma_j = \sqrt{\frac{1}{N} \sum_{i=1}^{N} (x_{ij} - \mu_j)^2}$. Next, we normalize each sample $\mathbf{x}_i$ with respect to its \( \ell_2 \)-norm as follows:
\begin{align*}
\mathbf{x}_i = \frac{\mathbf{x}_i}{\max(1, \|\mathbf{x}_i\|_2/c)},
\end{align*}
where $c$ is the $\ell_2$ norm clipping parameter~\cite{abadi2016deep}. We denote the design matrix with the preprocessed data samples with $\mathbf{X}$. Note that, the labels $y_i$ are converted to one-hot encoded vectors, and the labels matrix is denoted as $\mathbf{Y} \in \mathbb{R}^{N \times K}$. Our goal now is to generate a synthetic dataset $\tilde{\mathbf{X}} \in \mathbb{R}^{T \times d_{x}}$ and corresponding one-hot-encoded labels $\tilde{\mathbf{Y}} \in \mathbb{R}^{T \times K}$.
Since there are $K$ classes, we intend to ensure that each class of the synthetic dataset will have approximately $T_k = \left\lfloor \frac{T}{K} \right\rfloor$ synthetic samples. 
For each class $k \in \{1, 2, \dots, K\}$, we uniformly randomly select $l$ data points (\emph{order of mixture}) without replacement from the subset $\tilde{\mathbf{X}}_k = \left\{ \mathbf{x}_i : y_i=k \right\}$. Let $\left\{ \mathbf{x}_{i_{1}}, \mathbf{x}_{i_{2}}, ..., \mathbf{x}_{i_{l}} \right\}$ be the randomly selected data points from $\mathbf{X}_k$. A synthetic sample $\tilde{\mathbf{x}}_t^{(k)}$ is generated as:
$$\tilde{\mathbf{x}}_t^{(k)} = \frac{1}{l} \sum_{j=1}^{l} \mathbf{x}_{i_j} + \mathbf{n}_x,\ \forall t \in \{1, 2, \ldots, T_k\},$$ 
where $\mathbf{n}_x \sim \mathcal{N}(0, \sigma_x^2 \mathbf{I}_{d_x})$ is a $d_x$-dimensional Gaussian random variable with covariance matrix $\sigma_x^2 \mathbf{I}_{d_x}$, and $\mathbf{I}_{d_x}$ is a $d_x \times d_x$ identity matrix.

Similarly, the corresponding synthetic label $\tilde{y}_t^{(k)}$ is computed in two steps. First, the mean of the corresponding one-hot encoded labels of the $l$ randomly selected data points are computed, and noise drawn from a Gaussian distribution is added to it as follows:
$$\tilde{y}_t^{\text{one-hot(k)}} = \frac{1}{l} \sum_{j=1}^{l} y_{i_j}^{\text{one-hot}} + \mathbf{n}_y,$$ 
where $\mathbf{n}_y \sim \mathcal{N}(0, \sigma_y^2 \mathbf{I}_K)$ is a $K$-dimensional Gaussian random variable with covariance matrix $\sigma_y^2 \mathbf{I}_K$, $\mathbf{I}_K$ is a $K \times K$ identity matrix, and $y_{i_j}^{\text{one-hot}}$ is the one-hot-encoded label vector corresponding to $\mathbf{x}_{i_j}$. Then these one-hot encoded labels are converted to integer labels as follows:
\begin{equation}
\tilde{y}_{\text{t}}^{\text{(k)}} = \arg\max_{i \in \{0, 1, \dots, K-1\}} \tilde{y}_{\text{t}}^{\text{one-hot(k)}}[i].
\end{equation}
\begin{figure}[H]
    \centering
    \includegraphics[width=0.9\textwidth]{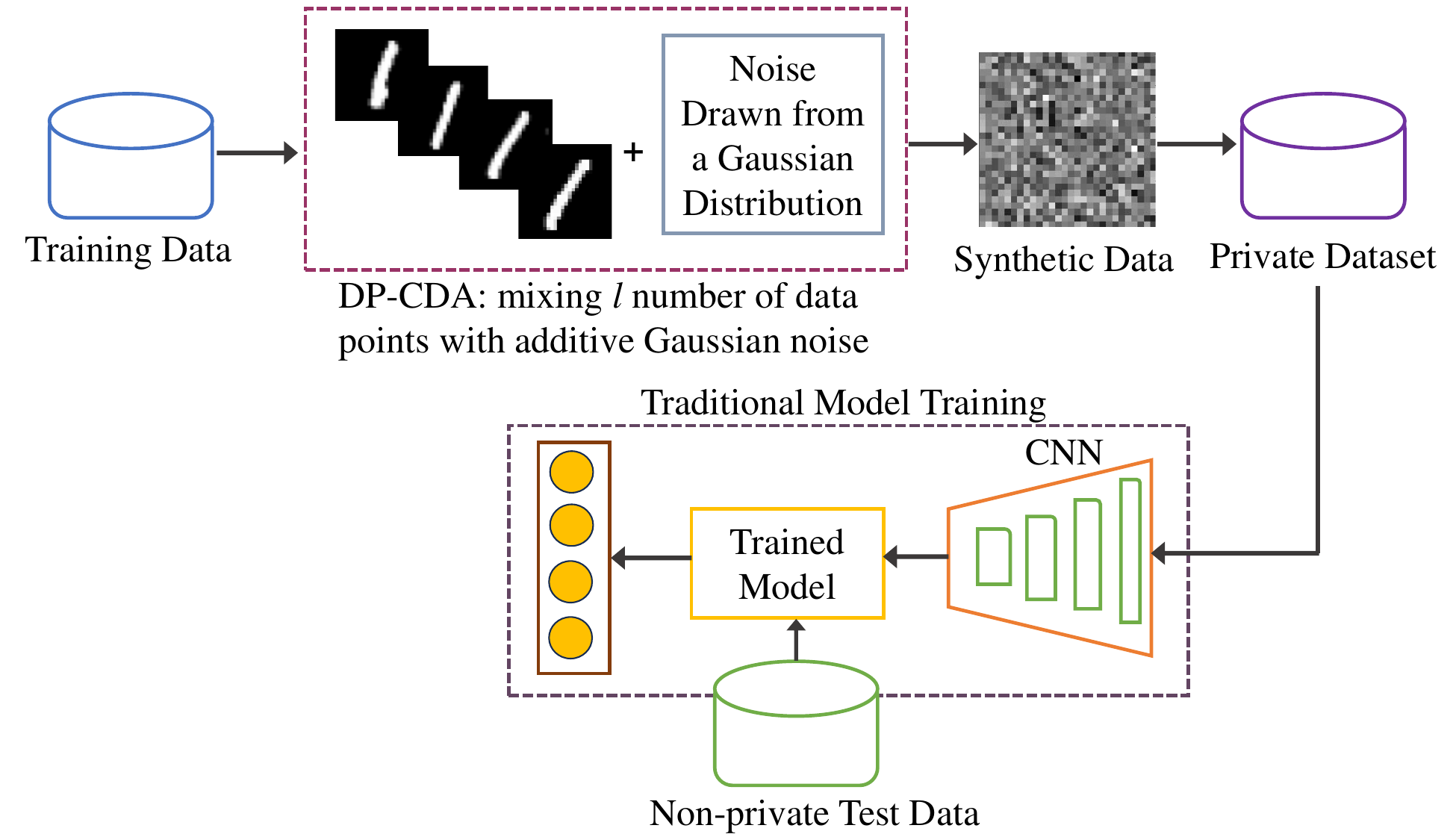}
    \caption{Flowchart of proposed synthetic data generation and subsequent model training}
    \label{fig:fig4}
\end{figure}
This process is repeated $T_k$ times for each class $k$ to form a class-balanced synthetic dataset $(\tilde{\mathbf{X}}, \tilde{\mathbf{y}})$. The complete synthetic dataset generation algorithm is shown in Algorithm~\ref{alg:dpcda}, and Figure~\ref{fig:fig4}. It is easy to show that the computational complexity of the proposed algorithm is $\mathcal{O}(Tld_x)$.

\begin{algorithm}[t]
\caption{DP-CDA: Privacy-Preserving Synthetic Data Generation}
\label{alg:dpcda}
\begin{algorithmic}[1]
\REQUIRE Dataset $\mathbb{D} = \{ (\mathbf{x}_i, y_i) \}_{i=1}^N$, number of synthetic samples $T$, order of mixture $l$, $\ell_2$ norm clipping parameter $c$, noise parameters $\sigma_x$ and $\sigma_y$
\ENSURE Synthetic dataset $(\tilde{\mathbf{X}}, \tilde{\mathbf{Y}})$
\STATE Formulate the design matrix $\hat{\mathbf{X}} \in \mathbb{R}^{N \times d_x}$ and label vector $\mathbf{y} \in \mathbb{R}^{N}$
\FOR{each feature $j \in \{1, 2, \ldots, d_x\}$}
    \STATE Compute $\mu_j = \frac{1}{N} \sum_{i=1}^{N} x_{ij}$ and $\sigma_j = \sqrt{\frac{1}{N} \sum_{i=1}^{N} (x_{ij} - \mu_j)^2}$
    \STATE Compute $x_{ij} = \frac{x_{ij} - \mu_j}{\sigma_j}$ $\forall i \in \{1, 2, \ldots, N\}$
\ENDFOR
\STATE Normalize $\mathbf{x}_i = \frac{\mathbf{x}_i}{\max(1, \|\mathbf{x}_i\|_2/c)}$ $\forall i \in \{1, 2, \ldots, N\}$
\STATE Convert $\mathbf{y}$ to one-hot encoded vectors $\mathbf{Y} \in \mathbb{R}^{N \times K}$, and set $T_k = \left\lfloor \frac{T}{K} \right\rfloor$
\FOR{each class $k \in \{1, 2, \ldots, K\}$}
    \FOR{each synthetic sample $t \in \{1, 2, \ldots, T_k\}$}
        \STATE Uniformly randomly select $l$ data points $\{\mathbf{x}_{i_1}, \mathbf{x}_{i_2}, \ldots, \mathbf{x}_{i_l}\}$ from $\{\mathbf{x}_i : y_i = k\}$\label{alg:dpcda:random-selection}
        \STATE Compute $\tilde{\mathbf{x}}_t^{(k)} = \frac{1}{l} \sum_{j=1}^{l} \mathbf{x}_{i_j} + \mathbf{n}_x$, where $\mathbf{n}_x \sim \mathcal{N}(0, \sigma_x^2 \mathbf{I}_{d_x})$ \label{alg:dpcda:noisex}
        \STATE Compute $\tilde{y}_t^{\text{one-hot(k)}} = \frac{1}{l} \sum_{j=1}^{l} y_{i_j}^{\text{one-hot}} + \mathbf{n}_y$, where $\mathbf{n}_y \sim \mathcal{N}(0, \sigma_y^2 \mathbf{I}_K)$\label{alg:dpcda:noisey}
        \STATE Convert to integer label $\tilde{y}_t^{(k)} = \arg\max_{i \in \{0, 1, \ldots, K-1\}} \tilde{y}_t^{\text{one-hot(k)}}[i]$
        \STATE Store $\tilde{\mathbf{x}}_t^{(k)}$ in $\tilde{\mathbf{X}}$ and $\tilde{y}_t^{(k)}$ in $\tilde{\mathbf{Y}}$
    \ENDFOR
\ENDFOR
\RETURN $(\tilde{\mathbf{X}}, \tilde{\mathbf{Y}})$
\end{algorithmic}
\end{algorithm}

\begin{theorem}[Privacy of DP-CDA (Algorithm~\ref{alg:dpcda})]
\label{theorem1}
Consider Algorithm~\ref{alg:dpcda} with input (privacy-sensitive) dataset $\mathbb{D} = \{ (\mathbf{x}_i, y_i) \}_{i=1}^N$, order of mixture $l$, $\ell_2$ norm clipping parameter $c$, and noise parameters $(\sigma_x, \sigma_y)$. Then Algorithm~\ref{alg:dpcda} releases a $(\epsilon, \delta)$-differentially private synthetic dataset $(\tilde{\mathbf{X}}, \tilde{\mathbf{Y}})$ consisting $T$ samples for any $0 < \delta < 1$ and $\alpha \geq 3$, where
\begin{align*}
    \epsilon &= \min_{\alpha \in \{3, 4, \ldots\}} T\varepsilon'(\alpha) + \frac{\log\frac{1}{\delta}}{\alpha-1}.
\end{align*}
Here, $\varepsilon'(\alpha) = \frac{1}{\alpha-1} \log \left( 1 + p^2 {\alpha \choose 2} \min \left\{ 4(e^{\varepsilon(2)} - 1), 2 e^{\varepsilon(2)} \right\} + 4 G(\alpha) \right)$,\\ where $\varepsilon(\alpha) = \frac{\alpha}{l^2}\left(\frac{2c^2}{\sigma_x^2} + \frac{1}{\sigma_y^2}\right)$, $G(\alpha) = \sum_{j=3}^{\alpha} p^{j} {\alpha \choose j} \sqrt{B(2 \left\lfloor \frac{j}{2} \right\rfloor) \cdot B(\left\lceil \frac{j}{2} \right\rceil)}$,\\ and $B(l) = \sum_{i=0}^{l} (-1)^{i} {l \choose i} e^{(i-1) \varepsilon(i)}$.
\end{theorem}

\begin{proof}
Consider a mechanism $\mathcal{M}$ that takes $l$ data points and is ($\alpha, \varepsilon(\alpha)$)-RDP. According to~\cite{wang19b}, a new mechanism $\mathcal{M'}$ that selects $l$ data points out of $N$ data points uniformly at random, and applies $\mathcal{M}$ is ($\alpha, \varepsilon'(\alpha)$)-RDP for any integer $\alpha \ge 2$, where 
\begin{align*}\label{eqn:ep_alpha}
\varepsilon'(\alpha) = \\
\frac{1}{\alpha-1} \log \left( 1 + p^2 {\alpha \choose 2} \min \left\{ 4(e^{\varepsilon(2)} - 1),  e^{\varepsilon(2)}\min\{2, (e^{\varepsilon(\infty)} - 1)\} \right\} + 4 G(\alpha) \right),
\end{align*}
$G(\alpha) = \sum_{j=3}^{\alpha} p^{j} {\alpha \choose j} \sqrt{B(2 \left\lfloor \frac{j}{2} \right\rfloor) \cdot B(\left\lceil \frac{j}{2} \right\rceil)}$, and $B(l) = \sum_{i=0}^{l} (-1)^{i} {l \choose i} e^{(i-1) \varepsilon(i)}$.

Now, computing the average of $l$ data points (Step~\ref{alg:dpcda:noisex} of Algorithm~\ref{alg:dpcda}) is $(\alpha,\frac{\alpha}{2(\frac{\sigma_x}{\Delta_x})^{2}})$-RDP~(Proposition~\ref{prop:rdp_gauss_mech}). Similarly, computing the average of $l$ one-hot-encoded vectors (Step~\ref{alg:dpcda:noisey} of Algorithm~\ref{alg:dpcda}) is $(\alpha,\frac{\alpha}{2(\frac{\sigma_y}{\Delta_y})^{2}})$-RDP~(Proposition~\ref{prop:rdp_gauss_mech}). Therefore, according to Proposition~\ref{prop:composition_rdp}, computing a synthetic data point is $(\alpha,\varepsilon(\alpha))$- RDP, where
\begin{align*}
    \varepsilon(\alpha) &= \frac{\alpha}{2}(\frac{\Delta ^{2}_{x}}{\sigma_{x^{2}}}+\frac{\Delta ^{2}_{y}}{\sigma_{y^{2}}}).
\end{align*}
Since in our setup $\|\mathbf{x}_i\|_2 \leq c$ and the rows of $\mathbf{Y}$ are one-hot-encoded, it is easy to show that $\Delta_x = \frac{2c}{l}$ and $\Delta_y = \frac{\sqrt{2}}{l}$. Therefore,
\begin{align*}
    \varepsilon(\alpha) &= \frac{\alpha}{l^{2}}(\frac{2c^{2}}{\sigma_x^2}+\frac{1}{\sigma_y^2}).
\end{align*}
Now, since we select $l$ data points randomly out of $N$ total data points at each step $t$ (see Step~\ref{alg:dpcda:random-selection} of Algorithm~\ref{alg:dpcda}), generation of the $t$-th synthetic sample is ($\alpha, \varepsilon'(\alpha)$)-RDP~\cite{wang19b}, where $\varepsilon'(\alpha)$ is given above. Simple algebra leads to the expression $\varepsilon'(\alpha) = \frac{1}{\alpha-1} \log \left( 1 + p^2 {\alpha \choose 2} \min \left\{ 4(e^{\varepsilon(2)} - 1), 2 e^{\varepsilon(2)} \right\} + 4 G(\alpha) \right)$. Since there are $T$ total synthetic samples are generated, the overall algorithm is $(\epsilon, \delta)$-differentially private for any $0 < \delta < 1$ according to Proposition~\ref{prop:rdp_dp}, where $\epsilon = \min_{\alpha \in \{3, 4, \ldots\}} T\varepsilon'(\alpha) + \frac{\log\frac{1}{\delta}}{\alpha-1}$.
\end{proof}
% \begin{Rem}
% As mentioned before, our algorithm closely follows the algorithm proposed in~\cite{lee2019synthesizing}. However, our privacy analysis is different from that presented in~\cite{lee2019synthesizing}. More specifically, the $\varepsilon(\alpha)$ as computed in~\cite{lee2019synthesizing} is $\varepsilon(\alpha) = \frac{\alpha}{2l^2}\left(\frac{d_x}{\sigma_x^2}+\frac{d_y}{\sigma_y^2}\right)$. Our analysis of Algorithm~\ref{alg:dpcda} is much tighter and avoids the ambient data dimension. More specifically, we achieve $\varepsilon(\alpha) = \frac{\alpha}{l^{2}}(\frac{2c^{2}}{\sigma_x^2}+\frac{1}{\sigma_y^2})$. As such, our privacy guarantee is stricter than the work of~\cite{lee2019synthesizing}, i.e., we achieve smaller $\epsilon$ for a given utility. This is intuitive, since the utility of an algorithm is dictated by the additive noise variance, whereas the level of privacy is dictated by $\varepsilon(\alpha)$.
% \end{Rem}
\begin{Rem}
As mentioned before, our algorithm closely follows the algorithm proposed in~\cite{lee2019synthesizing}. However, our privacy analysis is different from that presented in~\cite{lee2019synthesizing}. More specifically, the $\varepsilon(\alpha)$ as computed in~\cite{lee2019synthesizing} is $\varepsilon(\alpha) = \frac{\alpha}{2l^2}\left(\frac{d_x}{\sigma_x^2}+\frac{d_y}{\sigma_y^2}\right)$, which explicitly depends on the input and output data dimensions \( d_x \) and \( d_y \). This dimensional dependence introduces a practical limitation: as the data dimension increases, achieving a target privacy budget requires injecting more noise, leading to notable degradation in model utility. In contrast, our analysis of Algorithm~\ref{alg:dpcda} eliminates this dependency on data dimensionality. The resulting bound is
$\varepsilon(\alpha) = \frac{\alpha}{l^2} \left( \frac{2c^2}{\sigma_x^2} + \frac{1}{\sigma_y^2} \right)$.
This formulation ensures that the privacy loss is independent of the ambient feature space dimensions, making it especially suitable for high-dimensional data settings. As such, our privacy guarantee is stricter than the work of~\cite{lee2019synthesizing}, i.e., we achieve smaller $\epsilon$ for a given utility. This is intuitive, since the utility of an algorithm is dictated by the additive noise variance, whereas the level of privacy is dictated by $\varepsilon(\alpha)$, which is crucial for real-world applications, where maintaining both strong privacy and high utility in high-dimensional datasets is a key challenge.
\end{Rem}

\begin{figure}[t]
    \centering
    \includegraphics[width=0.8\textwidth]{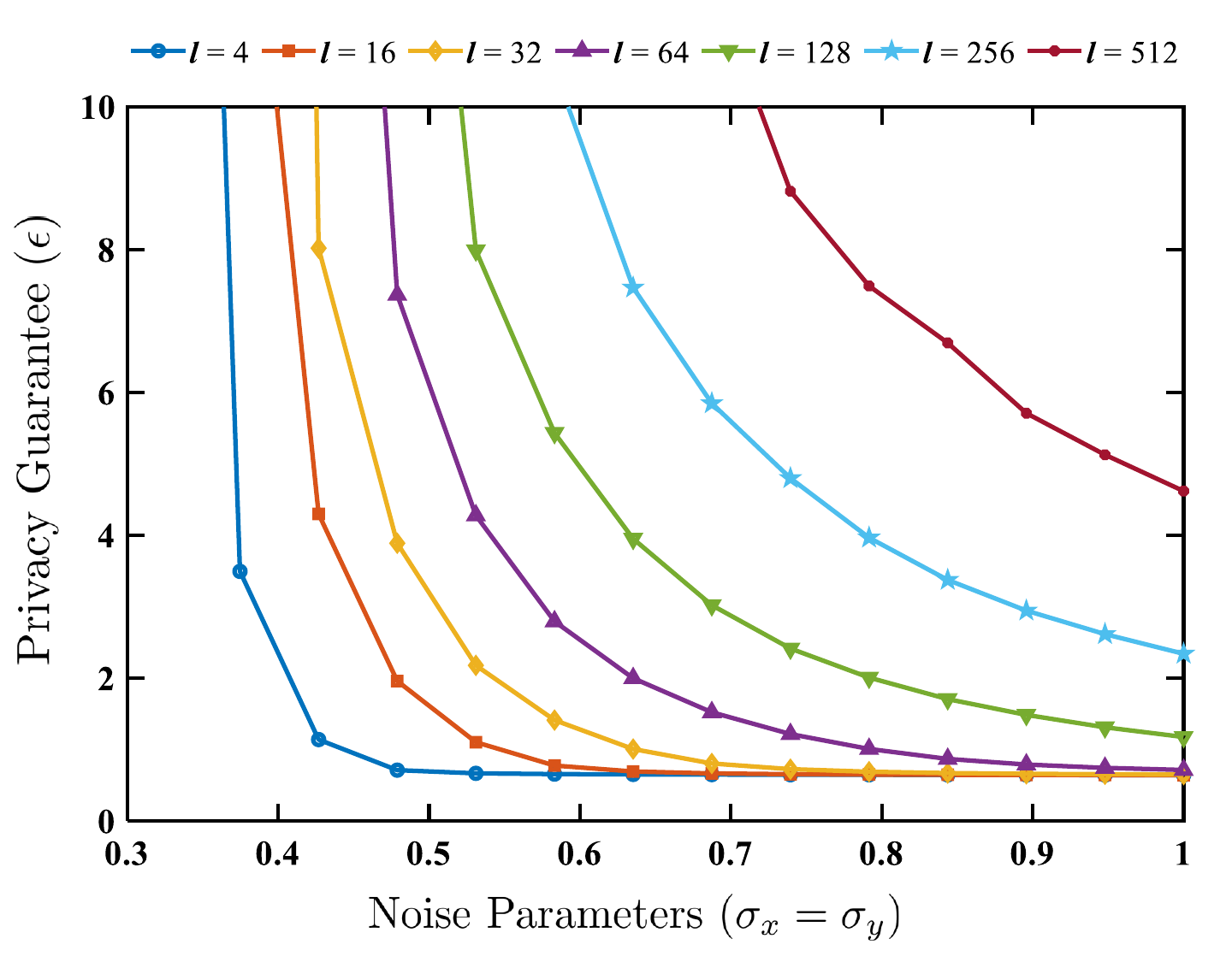}
    \caption{Privacy guarantee as a function of noise parameters.}
    \label{fig:fig3}
\end{figure}

\begin{Rem}
In Figure \ref{fig:fig3}, we illustrate how the overall privacy budget $\epsilon$ varies with the noise parameters ($\sigma_x, \sigma_y$) for different values of the order of mixture $l$ for a constant $N=$ 60,000. This plot clearly indicates that, for a fixed noise parameter, the privacy risk increases as \( l \) increases. That is, by selecting a larger value of \( l \), a smaller amount of noise will suffice to achieve a given level of privacy guarantee. Additionally, as depicted in Figure~\ref{fig:fig2}, the utility depends on $l$ as well. As such, there exists an \emph{optimal} value of $l$ for a given dataset and noise parameters that will result in the best utility and privacy guarantee. We will demonstrate this in Section~\ref{sec:res}. Note that, optimal values of $l$, and $c$ is dependent on the dataset and model architecture.
\end{Rem}

\section{Experimental Results}\label{sec:res}
\noindent\textbf{Experimental Setup. }In this section, we empirically validate the performance of our algorithm and the claim that our tighter privacy analysis of Algorithm~\ref{alg:dpcda} provides a stricter privacy guarantee than~\cite{lee2019synthesizing} for a given utility. For generating synthetic dataset according to Algorithm~\ref{alg:dpcda}, we consider three diverse and well-known datasets -- the MNIST dataset~\cite{lecun1998gradient}, the FashionMNIST dataset~\cite{xiao2017fashion}, and the CIFAR-10 dataset~\cite{krizhevsky2009learning}. Note that the proposed DP-CDA framework conforms to the standard $(\epsilon, \delta)$-Differential Privacy formulation, as stated in Theorem~\ref{theorem1}. To ensure an unbiased evaluation of data utility, the synthetic datasets generated by DP-CDA were used to train models, which were then tested on the original (non-synthetic) test sets of MNIST, FashionMNIST, and CIFAR-10. These experiments were conducted under different privacy parameters to assess the effectiveness of the proposed framework in maintaining a favorable privacy-utility balance. As mentioned before, we intend to train a neural network with the synthetic data generated by Algorithm~\ref{alg:dpcda}, and then evaluate the performance of the model on the test partition of the corresponding real dataset for different order of mixture $l \in \{1, 2, 4, \ldots, 512\}$ and noise levels $(\sigma_x, \sigma_y)$. From Theorem~\ref{theorem1}, we note that the overall privacy budget $\epsilon$ depends on both $l$ and $(\sigma_x, \sigma_y)$. For the neural network model, we consider a shallow convolutional neural network (CNN). The details of the architecture of the CNN are provided in Table~\ref{table:cnn-model}. 

Note that our model includes batch normalization and dropout to enhance training stability and reduce overfitting, i.e., robust for image classification tasks. We chose to implement the Adam optimizer, setting an initial learning rate of 0.001, and employed a learning rate scheduler to adjust the learning rate as needed. As mentioned before, we recorded the accuracies on the original test partition of the base dataset, which is unmixed and noise-free.\\

\begin{table}[t]
\centering
\caption{Details/parameters of the model under consideration}
\label{table:cnn-model}
\resizebox{1.0\textwidth}{!}{
\begin{tabular}{l l l}
\hline
\hline
\textbf{Layer Type}      & \textbf{Parameters}                                           & \textbf{Output Shape}    \\ \hline
Input                    & -                                                            & (1, 28, 28)              \\ 

Conv2D                   & Filters=32, Kernel=5x5, Stride=1, Padding=2                  & (32, 28, 28)             \\ 

ReLU Activation          & -                                                            & (32, 28, 28)             \\ 

BatchNorm2D              & Num Features=32                                             & (32, 28, 28)             \\ 

MaxPooling2D             & Pool Size=2x2, Stride=2                                      & (32, 14, 14)             \\ 

Conv2D                   & Filters=64, Kernel=3x3, Stride=1, Padding=1                  & (64, 14, 14)             \\ 

ReLU Activation          & -                                                            & (64, 14, 14)             \\ 

BatchNorm2D              & Num Features=64                                             & (64, 14, 14)             \\ 

MaxPooling2D             & Pool Size=2x2, Stride=2                                      & (64, 7, 7)               \\ 

Flatten                  & -                                                            & (3136)                   \\ 

Fully Connected (FC1)    & Input=3136, Output=100                                       & (100)                    \\ 

ReLU Activation          & -                                                            & (100)                    \\ 

Dropout                  & Probability=0.5                                             & (100)                    \\ 

Fully Connected (FC2)    & Input=100, Output=100                                        & (100)                    \\ 

ReLU Activation          & -                                                            & (100)                    \\ 

Dropout                  & Probability=0.5                                             & (100)                    \\ 

Fully Connected (Output) & Input=100, Output=10                                         & (10)                     \\ 
\hline
\end{tabular}
}
\end{table}

\begin{table}[htbp]
\centering
\caption{Average test accuracy $\pm$ standard deviation for each dataset with different $l$ and privacy budgets ($\epsilon$)}
\label{tab:1}
\resizebox{0.8\textwidth}{!}{
\begin{tabular}{ccccc}
\hline
\hline
\textbf{Dataset} & $l$ & $\epsilon = \infty$ & $\epsilon = 20$ & $\epsilon = 10$ \\
\hline
\multirow{2}{*}{MNIST} & 1 & 0.9885
 ± 0.0006 & 0.1080 ± 0.0055 & 0.1081 ± 0.0055 \\
                       & 2 & 0.9746
 ± 0.0011 & 0.6722 ± 0.1153 & 0.7076 ± 0.0480 \\
                       & 4 & 0.9462 ± 0.0031 & \textbf{0.7808 ± 0.0149} & \textbf{0.7807 ± 0.0188} \\
                       & 8 & 0.9045 ± 0.0064 & 0.7802 ± 0.0011 & 0.7673 ± 0.0156 \\
                       & 16 & 0.8607 ± 0.0052 & 0.7691 ± 0.0183 & 0.7585 ± 0.0165 \\
                       & 32 & 0.8237 ± 0.0106 & 0.7577 ± 0.0156 & 0.6738 ± 0.1929 \\
                       & 64 & 0.7962 ± 0.0091 & 0.7175 ± 0.0254 & 0.6319 ± 0.1777 \\
                       & 128 & 0.7829 ± 0.0176 & 0.6656 ± 0.0472 & 0.5431 ± 0.2180 \\
                       & 256 & 0.7693 ± 0.0189 & 0.3620 ± 0.2573 & 0.3403 ± 0.2010 \\
                       & 512 & 0.7603 ± 0.0181 & 0.2597 ± 0.1351 & 0.1575 ± 0.0877 \\
\hline
\multirow{2}{*}{FashionMNIST} & 1 & 0.9064 $\pm$ 0.0017 & 0.3886 $\pm$ 0.1910 & 0.3619 $\pm$ 0.1451 \\
                              & 2 & 0.8851 $\pm$ 0.0026 & 0.6712 $\pm $0.0122 & 0.6655 $\pm$ 0.0105 \\
                        & 4 & 0.8454 $\pm$ 0.0033 & \textbf{0.6791 $\pm$ 0.0061} & \textbf{0.6748 $\pm$ 0.0052} \\
                       & 8 & 0.7881 $\pm$ 0.0031 & 0.6719 $\pm$ 0.0059 &  0.6703 $\pm$ 0.0095 \\
                       & 16 & 0.7370 $\pm$ 0.0091 & 0.6652 $\pm$ 0.0058 & 0.6653 $\pm$ 0.0093 \\
                       & 32 & 0.6970 $\pm$ 0.0118 & 0.6542 $\pm$ 0.0159 & 0.6629 $\pm$ 0.0049 \\
                       & 64 & 0.6579 $\pm$ 0.0162 & 0.6538 $\pm$ 0.0116 & 0.6498 $\pm$ 0.0071 \\
                       & 128 & 0.6435 $\pm$ 0.0209 & 0.6532 $\pm$ 0.0107 & 0.6399 $\pm$ 0.0184 \\
                       & 256 & 0.6273 $\pm$ 0.0209 & 0.5966 $\pm$ 0.2124 & 0.6203 $\pm$ 0.0146 \\
                       & 512 & 0.6181 $\pm$ 0.0221 & 0.5243 $\pm$ 0.0149 & 0.4574 $\pm$ 0.1811 \\
\hline
\multirow{2}{*}{CIFAR10} & 1 & 0.7210 $\pm$ 0.0035 & 0.1001 $\pm$ 0.0000 & 0.1001 $\pm$ 0.0000 \\
                         & 2 & 0.6816 $\pm$ 0.0057 & 0.1994 $\pm$ 0.0602 & 0.1922 $\pm$ 0.0523 \\
                    & 4 & 0.6241 $\pm$ 0.0068 & \textbf{0.2315 $\pm$ 0.0449} & \textbf{0.2313 $\pm$ 0.0455} \\
                       & 8 & 0.5517 $\pm$ 0.0070 & 0.2258 $\pm$ 0.0435 &  0.1828$ \pm$ 0.0681 \\
                       & 16 & 0.4503 $\pm$ 0.0105 & 0.1761 $\pm$ 0.0657 & 0.1860 $\pm$ 0.0580 \\
                       & 32 & 0.3758 $\pm$ 0.0076 & 0.1938 $\pm$ 0.0542 & 0.2026 $\pm$ 0.0523 \\
                       & 64 & 0.3376 $\pm$ 0.0108 & 0.2001 $\pm$ 0.0515 & 0.1854 $\pm$ 0.0570 \\
                       & 128 & 0.3031 $\pm$ 0.0077 & 0.1579 $\pm$ 0.0498 & 0.1341 $\pm$ 0.0457 \\
                       & 256 & 0.2758 $\pm$ 0.0100 & 0.1449 $\pm$ 0.0455 & 0.1390 $\pm$ 0.0378 \\
                       & 512 &  0.2574 $\pm$ 0.0068 & 0.1146 $\pm$ 0.0299 & 0.1061 $\pm$ 0.0179 \\
\hline
\end{tabular}
}
\end{table}

\noindent\textbf{Effect of the order of mixture $(l)$ on utility. }In Table~\ref{tab:1}, we present the average accuracy and standard deviation (over 10 independent runs) for each of the aforementioned base datasets with different $l$ and privacy budgets $\epsilon$. Note that, for each \( l \), we set the noise parameters \( \sigma_x \) and \( \sigma_y \) in such a way that we achieve the targeted $\epsilon$. The column \( \epsilon = \infty \) is the non-privacy-preserving case $(\sigma_x = 0, \sigma_y=0)$. As such, $l=1$ and $\epsilon=\infty$ indicate conventional non-privacy-preserving training with the base dataset's training partition. From Table~\ref{tab:1}, we can observe that as $l$ increases for the non-private case, the utility decreases. This trend is observed across all three base datasets that we considered. In contrast, when examining the private dataset, we noticed a different trend. According to Theorem~\ref{theorem1}, the amount of noise required to guarantee a given privacy level decreases as the order of mixture $l$ increases. However, large $l$ values cause \emph{over-mixing}, which can dilute class-specific information and degrade model performance. Consequently, there exists an optimal $l^*$ that maximizes the utility of the generated synthetic dataset. To empirically demonstrate the existence of such $l^*$, we repeated the classification experiments using synthetic datasets generated by DP-CDA for different values of $l$. For each choice of $l$, we carefully calibrated the noise parameters ($\sigma_x, \sigma_y$) to ensure a fixed privacy level, i.e., constant ($\epsilon, \delta$) values. We then trained the CNN model on each synthetic dataset and evaluated the performance on the original (non-synthetic) test set. In our experiments, for all three base datasets and different $\epsilon$ levels, we found that the model obtained the best test accuracy for \( l^* = 4 \). The results, summarized in Table~\ref{tab:1} (columns $\epsilon=20$ and $\epsilon=10$), clearly demonstrate that the test accuracy varies with the choice of $l$, with the highest accuracy consistently achieved at $l^*=4$ across all datasets and privacy budgets. This optimal $l^*$ arises from a trade-off between privacy-induced noise and the degree of data mixing. Smaller $l$ values limit the averaging effect, resulting in higher sensitivity and greater noise injection, whereas larger $l$ values cause excessive averaging that blurs class-discriminative features. Hence, an intermediate $l^*$ offers the best balance, minimizing injected noise while maintaining sufficient class separability. Although $l^*$ may vary slightly across datasets depending on data complexity and dimensionality, our experiments suggest a stable pattern in which $l^* \approx 4$ provides consistently high utility under equivalent privacy constraints.

% According to Theorem~\ref{theorem1}, the amount of noise decreases as \( l \) increases for a given privacy level. Consequently, there exists an optimal value of \( l \) that maximizes utility. Through extensive experimentation, we observed a similar trend across three different datasets. As shown in Figure~\ref{fig:fig2}, there is a specific order of mixture $l^*$ for which the accuracy of the model reaches a maximum. In our experiments, for all three base datasets and different $\epsilon$ levels, we found that the model obtained the best test accuracy for \( l^* = 4 \).
\begin{figure}[ht]
    \centering
    \includegraphics[width=0.8\textwidth]{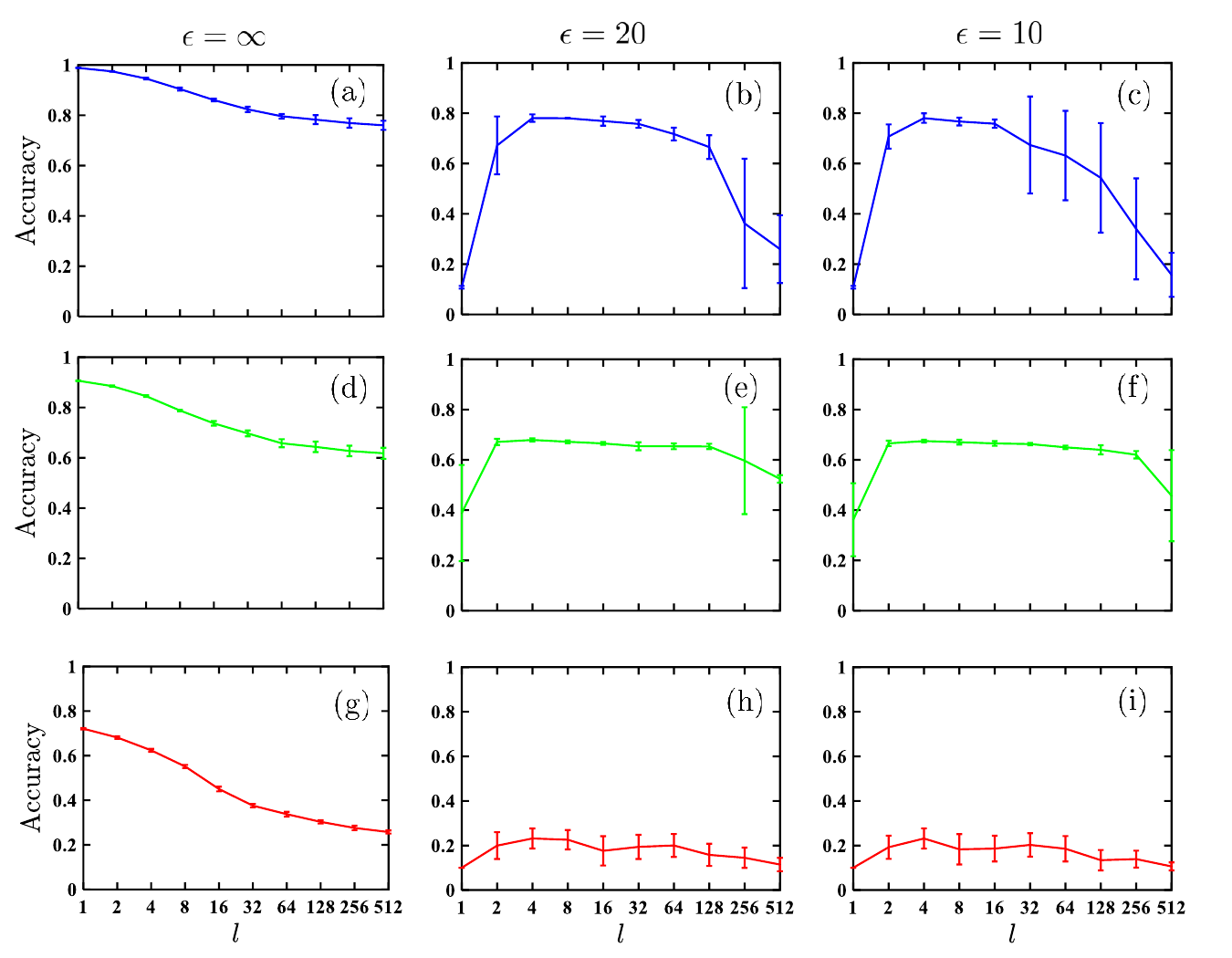}
    \caption{Average classification accuracy across varying mixture degrees $l$ for the MNIST, FashionMNIST, and CIFAR-10 datasets. The error bars indicate the standard deviation of run-to-run variations. Subfigures represent results under different privacy budgets: (a) MNIST, $\epsilon = \infty$; (b) MNIST, $\epsilon = 20$; (c) MNIST, $\epsilon = 10$; (d) FashionMNIST, $\epsilon = \infty$; (e) FashionMNIST, $\epsilon = 20$; (f) FashionMNIST, $\epsilon = 10$; (g) CIFAR-10, $\epsilon = \infty$; (h) CIFAR-10, $\epsilon = 20$; (i) CIFAR-10, $\epsilon = 10$.}
    \label{fig:fig5}
\end{figure}
For improved clarity and interpretation, we present the average classification accuracy across varying mixture degrees $l$ for the three datasets in Figure~\ref{fig:fig5}. As previously described, each experiment was repeated ten times for every combination of mixture degree $l$ and privacy budget $\epsilon$, to account for variability and provide a robust statistical evaluation. As such, the error bars in Figure~\ref{fig:fig5} indicate the standard deviation of the run-to-run variations. From the plots, it is evident that in the non-private setting ($\epsilon = \infty$), the classification performance is highly consistent across all three datasets. This is reflected by the minimal error bars in Figures~\ref{fig:fig5}(a), (d), and (g).

In the case of moderate privacy ($\epsilon = 20$), we observe a noticeable increase in variability for the MNIST and FashionMNIST datasets, particularly for higher $l$ values, such as $l = 256$, as shown in Figures~\ref{fig:fig5}(b) and (e). Interestingly, the CIFAR-10 dataset remains relatively stable under the same privacy budget, exhibiting narrower error bars across all values of $l$, as shown in Figure~\ref{fig:fig5}(h). Under a stronger privacy constraint ($\epsilon = 10$), the MNIST dataset displays greater variability, particularly at higher values of $l$, indicating that the utility of synthetic data becomes more sensitive in this regime. In contrast, both FashionMNIST and CIFAR-10 continue to demonstrate more stable behavior with relatively smaller error bars, as shown in Figures~\ref{fig:fig5}(f) and (i). Overall, the error bars remain reasonably narrow in most scenarios, suggesting that our synthetic data generation framework is robust and yields consistent performance across multiple runs, even under varying privacy constraints and mixture degrees.\\

\noindent \textbf{Performance of DP-CDA on Tabular Data.} To evaluate the applicability of our proposed DP-CDA algorithm on real-world tabular data, we utilize the UCI-Adult dataset~\cite{adult_2}, which contains privacy-sensitive information such as income, gender, marital status, etc. This dataset contains some categorical features that are converted into a one-hot encoded format, and both training and test datasets are standardized using z-score normalization followed by $\ell_2$ norm clipping. Synthetic data is generated using our DP-CDA method by averaging $l$ randomly selected samples from the same class and adding Gaussian noise. To preserve the categorical semantics after mixing, each one-hot encoded segment is re-encoded using an \texttt{argmax}-based decoding followed by re-encoding. Finally, a decision tree classifier is trained on the synthetic data, and its performance is evaluated on the non-synthetic test set to assess the utility of the generated dataset.
\begin{table}[t]
\centering
\caption{Average test accuracy $\pm$ standard deviation for UCI-Adult dataset~\cite{adult_2} with different $l$ and privacy budgets ($\epsilon$)}
\label{tab:res:adult}
\begin{tabular}{cccc}
\hline
\hline
$l$   & $\epsilon=\infty$           & $\epsilon=20$             & $\epsilon=10$             \\
\hline
1   & 0.7030 ± 0.0169 & 0.6398 ± 0.0164 & 0.6517 ± 0.0198 \\
2   & 0.6084 ± 0.0302 & 0.5750 ± 0.0259 & 0.5589 ± 0.0193 \\
4   & 0.6582 ± 0.0403 & 0.5942 ± 0.0307 & 0.5855 ± 0.0251 \\
8   & 0.7115 ± 0.0181 & 0.6497 ± 0.0302 & 0.6344 ± 0.0435 \\
16  & 0.6941 ± 0.0264 & 0.7278 ± 0.0367 & 0.7114 ± 0.0263 \\
32  & 0.7433 ± 0.0281 & 0.7595 ± 0.0150 & 0.7776 ± 0.0078 \\
64  & 0.7531 ± 0.0185 & \textbf{0.7866 ± 0.0103} & \textbf{0.7821 ± 0.0143} \\
128 & 0.7167 ± 0.0297 & 0.7666 ± 0.0179 & 0.7634 ± 0.0221 \\
256 & 0.6657 ± 0.0403 & 0.7363 ± 0.0255 & 0.7561 ± 0.0154 \\
512 & 0.6950 ± 0.0623 & 0.7129 ± 0.0034 & 0.7135 ± 0.0035 \\
\hline
\end{tabular}
\end{table}
Table~\ref{tab:res:adult} summarizes the performance of our DP-CDA algorithm on the UCI-Adult dataset across different privacy budgets and mixture orders $l$. The results clearly indicate that the choice of $l$ significantly influences the trade-off between utility and privacy. In particular, for lower values of $l$, the model suffers from reduced accuracy due to insufficient averaging and higher variance. As $l$ increases, the utility improves consistently, reaching peak performance at $l=64$, where the algorithm achieves 78.66\% accuracy for $\epsilon=20$ and 78.21\% for $\epsilon=10$. These results highlight the ability of our algorithm to maintain strong predictive performance even under strict privacy constraints. Furthermore, the stability and utility observed on this real-world and privacy-sensitive tabular dataset demonstrates that DP-CDA is not only theoretically sound but also practically applicable across diverse data modalities beyond images.\\

\noindent\textbf{Performance Comparison With Existing Works.} Finally, we compare our proposed privacy analysis and classification accuracy with those of the existing approaches in Table \ref{tab:2}. The first method is random projection~\cite{xu2017dppro}, which involves drawing a random projection matrix, followed by a transformation and the addition of Gaussian noise. This method achieves an accuracy close to 10\%. Similar results can be observed with the local perturbation algorithm, which applies the Gaussian mechanism to each data point. For these two algorithms, the outcomes are nearly equivalent to random guessing. 
\begin{table}[t]
\centering
\caption{Performance comparison of the proposed algorithm with existing approaches in terms of privacy, utility (accuracy), and computational complexity}
\label{tab:2}
\begin{threeparttable} % Begin the threeparttable environment
\resizebox{\textwidth}{!}{ % Resize table to fit within the text width
\begin{tabular}{ccc|cc|cc|c}
\hline
\hline
\multirow{2}{*}{\textbf{Algorithm}} & \multicolumn{2}{c|}{\textbf{MNIST}}         & \multicolumn{2}{c|}{\textbf{FashionMNIST}}  & \multicolumn{2}{c|}{\textbf{CIFAR10}}       & \multirow{2}{*}{\textbf{Comp. Complexity}} \\ \cline{2-7}
 &
  \multicolumn{1}{c}{$\epsilon$ = 20} &
  $\epsilon$ = 10 &
  \multicolumn{1}{c}{$\epsilon$ = 20} &
  $\epsilon$ = 10 &
  \multicolumn{1}{c}{$\epsilon$ = 30} &
  $\epsilon$ = 20 &  \\ \hline
Random Projection~\cite{xu2017dppro}          & \multicolumn{1}{c}{0.104} & 0.100 & \multicolumn{1}{c}{-}     & -     & \multicolumn{1}{c}{0.101} & -     & - \\ 
Local Perturbation~\cite{agrawal2000privacy}         & \multicolumn{1}{c}{0.098} & 0.098 & \multicolumn{1}{c}{-}     & -     & \multicolumn{1}{c}{0.100} & -     & - \\ 
DP-Mix~\cite{lee2019synthesizing}                     & \multicolumn{1}{c}{0.800} & 0.782 & \multicolumn{1}{c}{-\tnote{a}}     & -\tnote{a}     & \multicolumn{1}{c}{0.269} & 0.244 & $\mathcal{O}(Tl(d_x + d_y))$ \\
G-PATE~\cite{long2021g}         & \multicolumn{1}{c}{-} & 0.510 & \multicolumn{1}{c}{-}     & 0.500     & \multicolumn{1}{c}{-} & -     & - \\
DP-Merf~\cite{harder2021dp}         & \multicolumn{1}{c}{-} & 0.630 & \multicolumn{1}{c}{-}     & 0.620     & \multicolumn{1}{c}{-} & -     & - \\
GS-WGAN~\cite{chen2020gs}         & \multicolumn{1}{c}{-} & 0.800 & \multicolumn{1}{c}{-}     & 0.650     & \multicolumn{1}{c}{-} & -     & - \\
DP-CDA (Ours)                                         & \multicolumn{1}{c}{0.796} & 0.795 & \multicolumn{1}{c}{0.685} & 0.680 & \multicolumn{1}{c}{0.276} & 0.276 & $\mathcal{O}(Tld_x)$ \\ \hline
\end{tabular}
} % End resize
\begin{tablenotes} % Begin tablenotes
\item[a] We note that the method by which the authors in~\cite{lee2019synthesizing} computed the overall $\epsilon$ resulted in numerical instability issues for the FashionMNIST dataset. This could also result from our hardware. Consequently, the experiment was not reproduced for FashionMNIST, as the noise variances could not be calculated. 
\end{tablenotes} % End tablenotes
\end{threeparttable} % End the threeparttable environment
\end{table}
% Additionally, the DP-Mix and DP-MERF methods perform similarly to our proposed methods. DP-MERF utilizes random feature representations of kernel mean embeddings, while DP-Mix employs a noise-added version of an averaged data point from randomly selected data points.
Additionally, DP-Mix~\cite{lee2019synthesizing} employs a noise-induced version of an averaged data point from randomly selected data points. This algorithm performs marginally better than our proposed method only when $\epsilon = 20$ on the MNIST dataset; however, in other cases, our proposed method achieves better utility for a given privacy level. We argue that it is evident from Table \ref{tab:2} that our proposed algorithm offers stricter privacy guarantees, reaching the targeted privacy level with a smaller amount of noise compared to the existing methods.

We have also incorporated comparisons against recent deep generative methods, such as G-PATE~\cite{long2021g}, DP-Merf~\cite{harder2021dp}, and GS-WGAN~\cite{chen2020gs}. These methods achieve reasonable accuracy (e.g., 0.800 for GS-WGAN on MNIST with $\epsilon = 10$), showcasing their strength in privacy-preserving data generation. However, while such generative methods can offer competitive utility, they suffer from high computational overhead and training instability. Moreover, the synthetic images they produce tend to exhibit strong visual similarity to the original data, which may raise additional privacy concerns in sensitive applications. These limitations reinforce the practicality and robustness of our proposed approach.

Additionally, we compared the computational complexity of the proposed DP-CDA algorithm with that of the existing data publishing algorithms. As shown in Table \ref{tab:2}, the computational complexity of the proposed DP-CDA algorithm is nearly identical to the DP-Mix algorithm. However, the proposed DP-CDA algorithm offers the advantages of higher utility and stricter privacy guarantees. Furthermore, to provide a clearer understanding of the computational aspects, we discuss the scalability and overhead of the proposed approach in large-scale applications. The time complexity of DP-CDA grows linearly with the number of iterations, order of mixture, and feature dimension, i.e., $\mathcal{O}(Tld_x)$, ensuring that the algorithm remains efficient for high-dimensional datasets. In practical implementations, both data mixing and noise addition steps are suitable for parallel processing, allowing substantial computational acceleration on modern multi-core or GPU architectures. A key reason why scalability does not degrade performance in the proposed DP-CDA is that our privacy analysis is independent of the data dimension. Increasing the dimensionality of the dataset does not change noise variance to maintain the same privacy guarantee. Therefore, the utility of the synthetic data and the downstream model performance remain unaffected. Even when the data size or feature dimension grows, the algorithm does not suffer any additional accuracy loss. This shows that DP-CDA scales to large and high-dimensional datasets while keeping its performance stable. Therefore, despite introducing additional steps for privacy preservation, the proposed DP-CDA algorithm maintains low computational overhead and remains suitable for real-world large-scale deployments. Therefore, despite introducing additional steps for privacy preservation, the proposed DP-CDA algorithm maintains low computational overhead and remains suitable for real-world large-scale deployments.

\section{Limitations and Future Work}\label{sec:limit}
While our work provides higher utility and stricter privacy guarantees compared to existing algorithms, it does have some limitations. One key limitation is that the utility on the CIFAR10 dataset for our proposed data publishing algorithm is lower than that of the other datasets. To address this issue with CIFAR10, DP-CDA could be integrated into a privacy-enhanced data distillation framework that aligns synthetic data generation with probability distribution matching. This approach could enable more effective knowledge transfer, while adhering to strict privacy constraints. 

Additionally, in this work, data aggregation and synthetic data generation are performed in a centralized manner -- that is, DP-CDA assumes access to the entire dataset in one location. This centralization restricts its applicability in some real-world scenarios, such as healthcare or finance, where data is often siloed across multiple nodes/clients due to privacy or regulatory constraints. To this end, an interesting direction for future work would be to implement DP-CDA in federated settings and to conduct a theoretical analysis to determine the optimal order of mixture. Specifically, to adapt the proposed DP-CDA framework for decentralized settings, each client can locally generate class-specific mixtures by combining $l$ samples and adding Gaussian noise to ensure local differential privacy. These locally synthesized samples can then be transmitted to a central server, which can aggregate them to form the global synthetic dataset. However, since each client holds fewer samples, the sensitivity of local computation increases, necessitating higher noise levels to maintain the same privacy guarantee -- potentially reducing utility compared to centralized implementations. To address this limitation, one may explore incorporating an anti-correlated noise addition scheme~\cite{imtiaz2021cape}, wherein structured noise components cancel out during aggregation. This approach is expected to improve the overall privacy-utility balance and bring the decentralized performance closer to that of centralized DP-CDA.

\section{Conclusion}\label{sec:conc}
In this work, we propose DP-CDA, a synthetic data generation algorithm, and perform a tighter accounting of the privacy guarantee. DP-CDA combines randomly selected data points from a specific class and induces additive noise for preserving privacy. We empirically demonstrate that given a particular dataset, the number of data points utilized for generating each synthetic data point has an \emph{optimal} value, which offers the best utility and dictates the overall privacy budget. Through our extensive experiments, we show that DP-CDA significantly outperforms the existing data publishing algorithms, while being less computationally complex and providing a strong privacy guarantee. Overall, we achieved a favorable balance between utility, complexity, and privacy with the proposed algorithm.

\section*{Conflict of Interest}
The authors have no conflicts of interest to disclose.

% \section*{Acknowledgement}\label{sec:acknow}
% All authors affirm that the conceptualization, methodology, analysis, and writing of this manuscript were carried out by the authors themselves. No generative AI tools were used to create, modify, or refine the scientific content, experimental design, results, or conclusions. Grammarly was used only for grammar, spelling, and phrasing corrections. All authors have thoroughly reviewed the entire manuscript.
% The authors express their gratitude to the authorities of the Department of Electrical and Electronic Engineering at Bangladesh University of Engineering and Technology (BUET) for their continuous support during this research.

\small
\bibliographystyle{ieeetr}
\bibliography{references}

\newpage
\appendix % Start of appendix
\section{Relevant Definitions and Theorems}
\label{sec:appen}
\begin{definition}[($\epsilon,\delta$)-DP \cite{dwork2006calibrating}]
  An algorithm $f : \mathcal{D} \mapsto \mathcal{T}$ provides ($\epsilon,\delta$)-differential privacy (($\epsilon,\delta$)-DP) if $\Pr(f(\mathbb{D})\in \mathcal{S}) \le \delta + e^\epsilon \Pr(f(\mathbb{D}^\prime)\in \mathcal{S} )$ for all measurable $\mathcal{S} \subseteq \mathcal{T} $ and for all neighboring datasets $\mathbb{D},\mathbb{D}^\prime \in \mathcal{D}$.
\end{definition}

Here, $\epsilon > 0,\ 0 < \delta < 1$ are the privacy parameters, and determine how the algorithm will perform in providing privacy/utility. The parameter $\epsilon$ indicates how much the algorithm's output deviates in probability when we replace one single person's data with another. The parameter $\delta$ indicates the probability that the privacy mechanism fails to give the guarantee of $\epsilon$. Intuitively, higher privacy results in poor utility. That is, smaller $\epsilon$ and $\delta$ guarantee more privacy, but lower utility. There are several mechanisms to implement DP: Gaussian \cite{dwork2006calibrating}, Laplace mechanism \cite{dwork2014algorithmic}, random sampling, and exponential mechanism~\cite{mcsherry2007mechanism} are well-known. Among the additive noise mechanisms, the noise standard deviation is scaled by the privacy budget and the sensitivity of the function.\\

\begin{definition}[$\ell_2$ sensitivity \cite{dwork2006calibrating}] The $\ell_2$- sensitivity of vector valued function $f(\mathbb{D})$ is $\Delta := \max_{\mathbb{D},\ \mathbb{D}^\prime} \Vert f(\mathbb{D}) - f(\mathbb{D}') \Vert_2$, where $\mathbb{D}$ and $\mathbb{D}^\prime$ are neighboring datasets.
\end{definition} 
The $\ell_2$ sensitivity of a function gives the upper bound of how much the function can change if one sample at the input is changed. Consequently, it dictates the amount of randomness/perturbation needed at the function's output to guarantee DP. In other words, it captures the maximum change in the output by changing any one user in the worst-case scenario.\\

\begin{definition}[Gaussian Mechanism \cite{dwork2014algorithmic}] Let $f: \mathcal{D} \mapsto \mathcal{R}^D$ be an arbitrary function with $\ell_{2}$ sensitivity $\Delta$. The Gaussian mechanism with parameter $\tau$ adds noise from $\mathcal{N}(0,\tau^2)$ to each of the $D$ entries of the output and satisfies $(\epsilon,\delta)$-DP for $\epsilon \in (0,1)$ and $\delta \in (0,1)$, if $\tau \geq \frac{\Delta}{\epsilon} \sqrt{2\log\frac{1.25}{\delta}}$.
\end{definition} 
Here, $(\epsilon,\delta)$-DP is guaranteed by adding noise drawn form $\mathcal{N}(0,\tau^2)$ distribution. Note that, there are an infinite number of combinations of $(\epsilon,\delta)$ for a given $\tau^2$~\cite{imtiaz2021cape}.\\

\begin{definition}[R\'enyi Differential Privacy (RDP) \cite{mironov2017renyi}] A randomized algorithm $f$ : $\mathcal{D} \mapsto \mathcal{T}$ is $(\alpha,\epsilon_{r})$-R\'enyi differentially private if, for any adjacent $\mathbb{D},\ \mathbb{D}^\prime \in \mathcal{D}$, the following holds: $D_{\alpha}(\mathcal{A}(\mathbb{D})\ || \mathcal{A}(\mathbb{D}^\prime)) \leq \epsilon_{r}$. Here, $D_{\alpha}(P(x)||Q(x))=\frac{1}{\alpha-1}\log\mathbb{E}_{x \sim Q} \Big(\frac{P(x)}{Q(x)}\Big)^{\alpha}$ and $P(x)$ and $Q(x)$ are probability density functions defined on $\mathcal{T}$.
\end{definition} 
% We use RDP for calculating the total privacy budget spent in our multi-stage algorithm. RDP provides a much simpler rule for calculating overall privacy risk $\epsilon$ that is shown to be tight~\citep{mironov2017renyi}.\\

\begin{prop}[From RDP to DP~\cite{mironov2017renyi}]\label{prop:rdp_dp}
If $f$ is an $(\alpha,\epsilon_r)$-RDP mechanism, it also satisfies $(\epsilon_r+\frac{\log 1/\delta}{\alpha-1},\delta)$-DP for any $0<\delta<1$.
\end{prop}

\begin{prop}[Composition of RDP~\cite{mironov2017renyi}]\label{prop:composition_rdp}
Let $f_1:\mathcal{D}\rightarrow \mathcal{R}_1$ be $(\alpha,\epsilon_1)$-RDP and $f_2:{\mathcal{R}_1} \times \mathcal{D} \rightarrow \mathcal{R}_2$ be $(\alpha,\epsilon_2)$-RDP, then the mechanism defined as $(X_1,X_2)$, where $X_1 \sim f_1(\mathbb{D})$ and $X_2 \sim f_2(X_1,\mathbb{D})$ satisfies $(\alpha,\epsilon_1 + \epsilon_2)$-RDP.
\end{prop}

\begin{prop}[RDP and Gaussian Mechanism~\cite{mironov2017renyi}]\label{prop:rdp_gauss_mech}
If $f$ has $\ell_2$ sensitivity 1, then the Gaussian mechanism $\mathcal{G}_{\sigma}f(\mathbb{D})=f(\mathbb{D})+e$ where $e \sim \mathcal{N}(0,\sigma^2)$ satisfies $(\alpha,\frac{\alpha}{2\sigma^2})$-RDP. Also, a composition of $T$ such Gaussian mechanisms satisfies $(\alpha,\frac{\alpha T}{2 \sigma^2})$-RDP.
\end{prop}
%
% --------------------------------------------- Start introduction
%
%%--------------------------------------- End references 

%
\end{document}